\documentclass{article} 
\usepackage[final]{colm2026_conference}

\usepackage{microtype}
\usepackage{hyperref}
\usepackage{url}
\usepackage{booktabs}
\usepackage{lineno}
\usepackage{graphicx}
\usepackage{wrapfig}

\definecolor{darkblue}{rgb}{0, 0, 0.5}
\hypersetup{colorlinks=true, citecolor=darkblue, linkcolor=darkblue, urlcolor=darkblue}

\usepackage{enumitem}
\usepackage{xspace}
\usepackage{bbm}
\usepackage{algorithm}
\usepackage{algpseudocode}
\usepackage{tcolorbox}
\tcbuselibrary{breakable,skins}
\newtcolorbox{promptbox}[1][]{
  colback=gray!5,
  colframe=gray!60,
  fontupper=\small\ttfamily,
  breakable,
  enhanced,
  left=4pt, right=4pt, top=4pt, bottom=4pt,
  #1
}

\algrenewcommand\algorithmicindent{1.0em}
\algnewcommand\Input{\item[\textbf{Input:}]}
\algnewcommand\Output{\item[\textbf{Output:}]}

\usepackage{amsmath,amsthm}

\newtheorem{theorem}{Theorem}
\newtheorem{proposition}[theorem]{Proposition}

\newtheorem{remark}{Remark}

\newcommand{\Ours}{\textsc{Slate}}
\newcommand{\SearchR}{\textsc{Search-R1}}
\newcommand{\se}{\mathcal{E}}

\newcommand{\E}{\mathbb{E}}
\newcommand{\Var}{\text{Var}}
\newcommand{\Cov}{\text{Cov}}

\title{Truncated Step-Level Sampling with Process Rewards for\\ Retrieval-Augmented Reasoning}

\author{Chris Samarinas, Haw-Shiuan Chang, and Hamed Zamani \\
Center for Intelligent Information Retrieval \\
University of Massachusetts Amherst \\
Amherst, MA, United States \\
\texttt{\{csamarinas, hschang, zamani\}@cs.umass.edu}}

\begin{document}


\ifcolmsubmission
\linenumbers
\fi

\maketitle

\begin{abstract}
Reinforcement learning has emerged as an effective paradigm for training large language models to interleave reasoning with search engine calls. However, existing approaches face a fundamental credit assignment problem: methods like \SearchR{} assign a single outcome reward to the entire multi-step trajectory, providing no signal about which reasoning or retrieval decisions were responsible for success or failure. Process-reward methods such as StepSearch introduce step-level supervision but still sample complete trajectories independently, so advantage estimates at any given step are contaminated by the randomness of all other steps.
We propose \Ours{} (\textbf{S}tep-\textbf{L}evel \textbf{A}dvantage estimation for \textbf{T}runcated \textbf{E}xploration), which addresses both problems through two complementary ideas. First, \emph{truncated step-level sampling} generates $k$ continuations from a shared prefix, isolating all variation to a single decision point. We prove this reduces the variance of advantage estimates by up to a factor of $T$ compared to full-trajectory sampling for $T$-step trajectories, the first formal variance guarantee for step-level RL in retrieval-augmented reasoning. Second, \emph{dense, decomposed process rewards} separately evaluate reasoning quality, query quality, and answer correctness on a ternary scale via an LLM judge, providing richer supervision than binary outcome signals or heuristic step-level scores.
Experiments on seven QA benchmarks show that \Ours{} consistently outperforms both sparse-reward and process-reward baselines, achieving a 7.0\% relative improvement over \SearchR{} on the 7B model and 30.7\% on the 3B model. Gains are largest on challenging multi-hop tasks, and ablations confirm that truncated sampling and dense rewards provide complementary benefits.
\end{abstract}

\section{Introduction}

Large language models have demonstrated remarkable capabilities in natural language understanding and generation~\citep{hendrycks2020measuring,clark2018think}. Despite these achievements, LLMs often struggle with complex reasoning tasks~\citep{wei2022chain} and lack access to up-to-date external knowledge~\citep{jin2024long}. Integrating search engines into the LLM reasoning loop, where the model interleaves its own chain-of-thought reasoning with external retrieval calls, has emerged as a promising paradigm for knowledge-intensive question answering~\citep{yao2023react,trivedi2022interleaving}. 

Reinforcement learning provides a natural framework for optimizing such systems. \SearchR{}~\citep{jin2025searchr1trainingllmsreason} pioneered RL training for LLMs that invoke search engines during multi-turn reasoning, using outcome-based exact-match rewards. While effective, this sparse reward suffers from \emph{the credit assignment problem}: a single binary signal after a multi-step trajectory cannot attribute success or failure to individual steps. Process-level supervision methods, including StepSearch~\citep{stepsearch2025} with heuristic step rewards and SWiRL~\citep{goldie2025synthetic} with LLM-judge binary step rewards, improve over outcome-only rewards, but both still sample \emph{complete} trajectories independently. As a result, when computing the advantage for step $t$'s action, the reward signal is contaminated by the randomness of all other steps: two trajectories that take the same action at step $t$ but differ at earlier or later steps will receive different total rewards, making it impossible to isolate whether step $t$'s action was genuinely good or bad.

More precisely, the advantage estimates $\hat{A}_i$ used in policy gradient methods like GRPO exhibit high variance because each estimate aggregates reward variation from all $T$ steps of the trajectory. High-variance policy gradients are a well-known obstacle in RL: they slow convergence, require larger batch sizes to stabilize, and can cause training instability or reward collapse~\citep{schulman2015high,schulman2017proximal}. Reducing this variance while preserving the ability to assign credit to individual actions is therefore a central challenge for effective step-level RL in retrieval-augmented reasoning.

In this paper, we propose \Ours{},\footnote{Code available at \url{https://github.com/algoprog/SLATE}.} Step-Level Advantage estimation for Truncated Exploration, which addresses these limitations using two complementary ideas:

\begin{itemize}[leftmargin=*]
    \item \textbf{Truncated Step-Level Sampling:} Instead of sampling $k$ fully independent trajectories, we sample $k$ truncated trajectories that share a common prefix $\tau_{<t}$ and differ only at step $t$. This allows GRPO-style group relative advantages to be computed at the step level, directly attributing rewards to the specific action that caused them. We formally prove this achieves a $T$-fold advantage variance reduction over full-trajectory sampling (Theorem~\ref{thm:var-reduction}), yielding lower-variance policy gradients---to our knowledge, the first formal variance guarantee for step-level RL in retrieval-augmented reasoning.

    \item \textbf{Dense, Decomposed Process Rewards:} Existing step-level rewards either rely on heuristics that require gold intermediate documents~\citep{stepsearch2025} or assign undifferentiated binary judgments that conflate distinct skills~\citep{goldie2025synthetic}. We introduce a decomposed reward that \emph{separately} evaluates reasoning quality, query quality, and answer correctness on a ternary scale $\{-1, 0, +1\}$, enabling the policy gradient to independently reinforce each competency required for effective retrieval-augmented reasoning. An LLM judge operationalizes this multi-criteria evaluation, with a reason-then-score protocol that substantially improves reliability.
\end{itemize}

Experiments across seven QA benchmarks demonstrate that \Ours{} consistently outperforms both sparse-reward methods (\SearchR{}) and process-reward methods (StepSearch), achieving an average EM of 0.461 on the 7B model (7.0\% relative improvement over \SearchR{}) with the largest gains on challenging multi-hop tasks. Gains are even more pronounced for smaller models, with a 30.7\% relative improvement on the 3B model, suggesting that dense step-level supervision is especially valuable when model capacity is limited. Ablations confirm that truncated sampling provides complementary gains above and beyond what dense rewards alone achieve, consistent with our theoretical analysis.

\section{Related Work}

Prior work has established that LLMs benefit substantially from access to external knowledge at inference time, motivating a broad framework of retrieval-enhanced machine learning~\citep{zamani2022retrieval}. Retrieval-augmented generation (RAG)~\citep{gao2023retrieval,lewis2020retrieval} integrates retrieved passages into LLM generation and has become the dominant approach for knowledge-intensive NLP tasks. However, single-turn retrieval struggles with complex questions that require iterative information gathering and multi-step reasoning~\citep{yang2018hotpotqa,trivedi2022interleaving}. Tool-use approaches such as Toolformer~\citep{schick2023toolformer} and Self-RAG~\citep{asai2024self} let models learn when and what to retrieve, but rely on supervised fine-tuning with expensive annotated trajectories.

\paragraph{Retrieval-Augmented Reasoning.}
To address the limitations of single-turn RAG, retrieval-augmented reasoning methods interleave chain-of-thought reasoning~\citep{wei2022chain} with retrieval calls, allowing the model to iteratively gather and synthesize information across multiple turns. Zero-shot approaches such as IRCoT~\citep{trivedi2022interleaving}, ReAct~\citep{yao2023react}, and Search-o1~\citep{li2025search} achieve this through prompting, but lack the ability to improve through training. The most effective optimization paradigm has been reinforcement learning, which has also driven recent progress in LLM reasoning more broadly~\citep{jaech2024openai,guo2025deepseek}, with policy gradient methods such as PPO~\citep{schulman2017proximal} and GRPO~\citep{shao2024deepseekmath} widely adopted. \SearchR{}~\citep{jin2025searchr1trainingllmsreason} pioneered RL training for LLMs that invoke search engines during multi-turn reasoning, using outcome-based exact-match rewards and retrieved token loss masking. Follow-up work includes R1-Searcher~\citep{song2025r1searcherincentivizingsearchcapability}, ReSearch~\citep{chen2025researchlearningreasonsearch}, and ZeroSearch~\citep{sun2025zerosearchincentivizesearchcapability}, all of which rely on sparse global rewards.

\paragraph{Process Rewards for Retrieval-Augmented RL.}
Process reward models have been explored for mathematical reasoning~\citep{lightman2023let,uesato2022solving}, but using step-level rewards as RL training signals has generally underperformed outcome-based rewards in math settings. In retrieval-augmented reasoning, however, step-level rewards are more naturally grounded (see Section~\ref{sec:experiments} for discussion). StepSearch~\citep{stepsearch2025} addresses the reward sparsity problem by introducing step-wise rewards based on information gain and redundancy penalties, but still samples complete trajectories and relies on access to ground-truth intermediate documents. SWiRL~\citep{goldie2025synthetic} uses an LLM judge~\citep{zheng2023judging} to provide step-level binary rewards, but operates \emph{offline} on pre-generated trajectories, so step-level advantages still conflate variation across different prefix histories with variation at the current step, and its undifferentiated binary judgments cannot disentangle distinct skills. Our work differs along three axes: (1)~truncated step-level sampling generates $k$ continuations from an \emph{identical} shared prefix, isolating variation to exactly one decision point with provable variance guarantees (Theorem~\ref{thm:var-reduction}); (2)~\emph{online} GRPO optimization avoids the distribution mismatch of offline approaches; and (3)~our decomposed ternary reward design provides richer supervision than heuristic step rewards~\citep{stepsearch2025} or binary LLM judgments~\citep{goldie2025synthetic}, without requiring ground-truth intermediate annotations.

\paragraph{Turn-Level Credit Assignment and Search Efficiency.}
Two concurrent lines of work also pursue finer-grained credit assignment for multi-turn agents. MT-GRPO~\citep{zeng2025mtgrpo} combines turn-level and outcome advantages within GRPO, but computing group advantages at every state requires tree-structured rollouts whose cost grows as $G^{K-1}$ with the horizon $K$, which the authors themselves describe as computationally prohibitive; their tractable GRPO instance is therefore restricted to two-turn interactions, with longer horizons handled by a critic-based PPO variant instead. Moreover, its turn-level advantages are still computed across full trajectories that diverge at every turn, so for turns $t > 1$ they conflate variation from different prefixes (the confound illustrated in Figure~\ref{fig:overview}), and its verifiable intermediate reward fires only when the gold answer string appears in the retrieved documents, tying query quality to retrieval outcome and to gold-evidence availability. \Ours{}'s truncated sampling obtains prefix-isolated credit assignment at cost \emph{linear} in the horizon ($T$ steps $\times$ $k$ single-step samples), applies at every step rather than only the first, imposes no fixed-turn-count constraint on the group, and comes with a formal variance guarantee (Theorem~\ref{thm:var-reduction}). HiPRAG~\citep{wu2025hiprag} instead targets search \emph{efficiency}: binary per-step over-/under-search detectors are aggregated into a single trajectory-level bonus under standard full-trajectory sampling, so it provides no prefix isolation and addresses an objective complementary to ours; truncated sampling is orthogonal to HiPRAG's reward and could in principle be combined with it.


\begin{figure}[t]
    \centering
    \includegraphics[width=\linewidth]{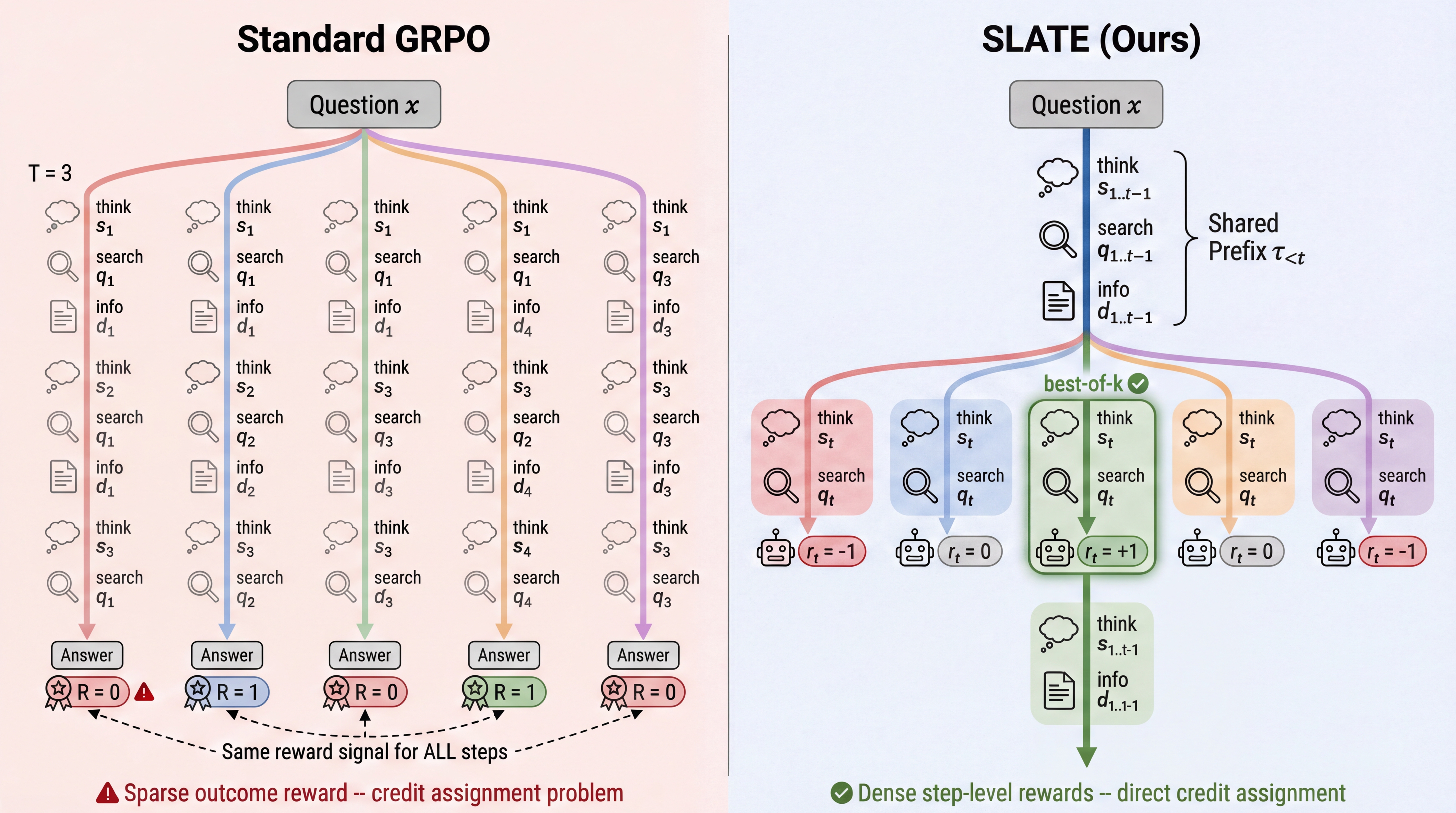}
    \caption{Comparison of GRPO (with full trajectory sampling) and our truncated step-level sampling. Existing process reward methods such as StepSearch and SWiRL follow the same full-trajectory sampling as standard GRPO (left) but inject step-level rewards, so step-level advantages still conflate variation from different prefix histories with variation at the current step. By contrast, our approach (right) fixes the prefix $\tau_{<t}$ and samples $k$ continuations from that shared prefix, isolating all variation to step $t$.}
    \label{fig:overview}
\end{figure}

\section{Methodology}\label{sec:method}

We present \Ours{}, a training framework for retrieval-augmented LLM reasoning that combines truncated step-level sampling with dense, decomposed process rewards. We build on the multi-turn search interaction framework of \SearchR{}~\citep{jin2025searchr1trainingllmsreason} and optimize using a modified GRPO objective. Figure~\ref{fig:overview} provides a high-level overview.

\subsection{Preliminaries: Retrieval-Augmented RL}\label{sec:prelim}

Following \SearchR{}, we model the search engine $\se$ as part of the environment. The LLM policy $\pi_\theta$ generates outputs interleaved with search engine calls, producing trajectories of the form:
\begin{equation}
    \tau = \underbrace{\langle\texttt{think}\rangle s_1 \langle\texttt{/think}\rangle}_{\text{reasoning}}\;
    \underbrace{\langle\texttt{search}\rangle q_1 \langle\texttt{/search}\rangle}_{\text{query}}\;
    \underbrace{\langle\texttt{info}\rangle d_1 \langle\texttt{/info}\rangle}_{\text{retrieval}}\;
    \cdots
\end{equation}
concluding with $\langle\texttt{answer}\rangle a \langle\texttt{/answer}\rangle$, where $s_t$ denotes the reasoning at step $t$, $q_t$ the search query, $d_t = \se(q_t)$ the retrieved documents, and $a$ the final answer. We denote the number of search steps as $T$. The standard RL objective with search is:
\begin{align}\label{eq:rl-search}
    \max_{\pi_\theta}\; & \E_{x \sim \mathcal{D},\, y \sim \pi_\theta(\cdot \mid x; \se)} \left[ r_\phi(x, y) \right] \nonumber - \beta \, \mathbb{D}_{\text{KL}}\!\left[\pi_\theta(y \mid x; \se) \,\|\, \pi_{\text{ref}}(y \mid x; \se)\right],
\end{align}
where $r_\phi(x, y)$ is the reward function and $\pi_{\text{ref}}$ is the reference policy.

\paragraph{Standard GRPO.}
In GRPO~\citep{shao2024deepseekmath}, $G$ complete trajectories $\{y_1, \ldots, y_G\}$ are sampled for each input $x$, and the advantage for trajectory $i$ is:
\begin{equation}\label{eq:grpo-adv}
    \hat{A}_i = \frac{R(y_i) - \mu_R}{\sigma_R + \epsilon},
\end{equation}
where $\mu_R = \frac{1}{G}\sum_{i=1}^{G} R(y_i)$, $\sigma_R = \sqrt{\frac{1}{G}\sum_{i=1}^{G}(R(y_i) - \mu_R)^2}$, and $R(y_i)$ is the trajectory-level reward. As discussed in Section~1, this suffers from poor credit assignment (a single scalar weights gradients for all $T$ steps) and high variance ($\hat{A}_i$ reflects variation across all steps).

\subsection{Truncated Step-Level Sampling}\label{sec:truncated}

Our key algorithmic novelty is \emph{truncated step-level sampling}: instead of sampling $k$ complete independent trajectories that may diverge from the very first step, we sample $k$ truncated trajectories that share a common prefix and differ only at the next reasoning step $t$. This design is motivated by a key theoretical result: we prove in Section~\ref{sec:theory} (Theorem~\ref{thm:var-reduction}) that truncated sampling reduces the variance of advantage estimates by up to a factor of $T$ compared to full-trajectory sampling, directly yielding lower-variance policy gradients and more stable training. Empirically, this translates to faster convergence and a higher reward ceiling (Section~\ref{sec:experiments}, Figure~\ref{fig:training-dynamics}).

\paragraph{Formal Definition.}
Let $\tau_{<t} = (s_1, q_1, d_1, \ldots, s_{t-1}, q_{t-1}, d_{t-1})$ denote the trajectory prefix up to (but not including) step $t$. At each step $t$, we generate $k$ candidate next-step actions by sampling from the policy conditioned on the shared prefix:
\begin{equation}\label{eq:sample-step}
    a_t^{(j)} = (s_t^{(j)}, q_t^{(j)}) \sim \pi_\theta\!\left(\cdot \mid x, \tau_{<t}\right), \quad j = 1, \ldots, k.
\end{equation}
Here, each $a_t^{(j)}$ consists of a reasoning step $s_t^{(j)}$ (the \texttt{<think>} block) followed by a search query $q_t^{(j)}$ (the \texttt{<search>} block), or alternatively a final answer $a^{(j)}$ (the \texttt{<answer>} block) if the model chooses to terminate. Each candidate action $a_t^{(j)}$ is then evaluated by the LLM-as-judge reward model (Section~\ref{sec:dense-reward}) to obtain a step-level reward $r_t^{(j)}$. The step-level group-relative advantage is:
\begin{equation}\label{eq:step-advantage}
    \hat{A}_t^{(j)} = \frac{r_t^{(j)} - \bar{r}_t}{\sigma_t + \epsilon},
\end{equation}
where $\bar{r}_t$ and $\sigma_t$ are the mean and standard deviation of rewards within the step-$t$ group.

\paragraph{Trajectory Construction.}
After computing advantages for all $k$ candidates at step $t$, we select the action to continue the trajectory. The selected action $a_t^{(j^*)}$ is appended to the prefix, the search engine is invoked to retrieve documents $d_t = \se(q_t^{(j^*)})$, and the process repeats at step $t+1$. The selection can follow different strategies:
\begin{itemize}[leftmargin=*]
    \item \textbf{Best-of-$k$}: $j^* = \arg\max_j r_t^{(j)}$ (pure exploitation).
    \item \textbf{Reward-weighted sampling}: $j^* \sim \text{softmax}(\hat{A}_t^{(1)}, \ldots, \hat{A}_t^{(k)} \,/\, \eta)$ with temperature $\eta$ (exploration-exploitation trade-off).
\end{itemize}
In our experiments we adopt \emph{reward-weighted sampling} (with temperature $\eta$), as it balances exploitation of high-reward actions with exploration of diverse reasoning paths, preventing the trajectory from collapsing to a single greedy mode early in training. We emphasize two points that define the procedure precisely. First, truncated sampling is applied at \emph{every} step $t = 1, \ldots, B$ of the trajectory, not at a selected subset: at each step the $k$ candidates are sampled, scored, and used for a gradient update, and exactly one is selected to extend the prefix. Second, each of the $k$ continuations is a \emph{single truncated step} of roughly $L/T$ tokens (where $L$ is the average full-trajectory length), not a complete rollout; the procedure never branches into $k$ full trajectories. The complete procedure is presented in Algorithm~\ref{alg:dense-r1} (Appendix~\ref{app:algorithm}).

\paragraph{Computational Cost.}
Because each candidate is a single step, per-question policy generation totals approximately $k L$ tokens ($k$ samples $\times$ $T$ steps $\times$ $L/T$ tokens each), the same order as standard GRPO's $G L$ at equal group sizes ($k = G$); the shared prefix is generated once and is not replicated across candidates. In fact, Proposition~\ref{cor:sample-eff} (Appendix~\ref{app:sample-efficiency}) shows that \emph{matching} GRPO's advantage variance requires only $G/T$ truncated samples per step, a $T$-fold reduction in total token generation. Truncated sampling is a training-time procedure only: at inference the trained policy generates a single trajectory exactly as in \SearchR{}, so deployment cost is unchanged. The genuine training overhead of \Ours{} is instead the LLM-judge calls (up to $k$ per step); these are short scoring generations that can be batched on separate inference hardware, but they do make training cost depend on the judge model, a dependence we return to in the Limitations (Appendix~\ref{app:limitations}).

\subsection{Dense, Decomposed Process Rewards}\label{sec:dense-reward}

A sparse binary outcome reward after a $T$-step trajectory cannot distinguish whether failure stems from poor reasoning, a bad query, or incorrect answer extraction. Prior step-level approaches partially address this: StepSearch~\citep{stepsearch2025} uses TF-IDF overlap with gold evidence documents, but requires ground-truth intermediate annotations, penalizes well-formed queries to sparse indices, and collapses all quality into a single scalar; SWiRL~\citep{goldie2025synthetic} uses an LLM judge with binary step rewards, but cannot distinguish harmful steps from mediocre ones or disentangle reasoning from query quality.

We introduce a \emph{decomposed} reward that separately evaluates each skill on a \emph{ternary} scale $\{-1, 0, +1\}$, enabling the policy gradient to independently reinforce or penalize reasoning, query formulation, and answer correctness. The design requires only the trajectory context and gold final answer, no intermediate annotations. An LLM judge (Appendix~\ref{app:prompts}) operationalizes the evaluation using a ``reason-then-score'' protocol, which we found substantially improves reward reliability. At each step $t$, the judge evaluates:

\paragraph{Thinking Reward.}
$r_{\text{think}}(s_t, \tau_{<t}) \in \{-1, 0, +1\}$ scores the reasoning step $s_t$ on five criteria (\textit{relevance}, \textit{clarity}, \textit{specificity}, \textit{progress}, \textit{faithfulness}).

\paragraph{Query Reward.}
$r_{\text{query}}(q_t, s_t, \tau_{<t}) \in \{-1, 0, +1\}$ scores the search query $q_t$ on five criteria (\textit{relevance}, \textit{specificity}, \textit{searchability}, \textit{alignment}, \textit{novelty}). Crucially, the query is evaluated \emph{before} observing retrieval results, so the reward reflects intrinsic query quality rather than retrieval nondeterminism.

\paragraph{Final Answer Reward.}
$r_{\text{answer}}(a, a_{\text{gold}}, \tau) \in \{-1, 0, +1\}$ scores whether the predicted answer $a$ conveys the same information as $a_{\text{gold}}$, distinguishing partially correct from fully incorrect answers and handling paraphrases.

\paragraph{Composite Step Reward.}
The total reward for action $a_t^{(j)} = (s_t^{(j)}, q_t^{(j)})$ at step $t$ is:
\begin{equation}\label{eq:step-reward}
    r_t^{(j)} = r_{\text{think}}(s_t^{(j)}, \tau_{<t}) + r_{\text{query}}(q_t^{(j)}, s_t^{(j)}, \tau_{<t}).
\end{equation}
When the model produces an answer at step $t$, the reward additionally includes the answer component and an early-termination bonus $\lambda \cdot (B - t) / B$ that encourages answering as soon as sufficient information is gathered (Appendix~\ref{app:early-termination}):
\begin{equation}\label{eq:final-reward}
    r_t^{(j)} = r_{\text{think}}(s_t^{(j)}, \tau_{<t}) + r_{\text{answer}}(a^{(j)}, a_{\text{gold}}, \tau) + \lambda \cdot \frac{B - t}{B},
\end{equation}
where $B$ is the maximum action budget and $\lambda \geq 0$ controls the bonus strength.

\paragraph{Reward Modeling, Not Distillation.}
Although the judge (Gemma3-27B) is larger than the policies we train, \Ours{} is on-policy RL in the standard RLHF/RLAIF mold~\citep{ouyang2022training,bai2022constitutional,lee2023rlaif}, in which the reward model is routinely as large as or larger than the policy: the judge never produces tokens for the policy to imitate; it only scores the policy's \emph{own} generations on a $\{-1, 0, +1\}$ scale, and gradients flow through the policy's samples via Eqs.~\ref{eq:masked-step-grpo}--\ref{eq:full-objective}. The judge also receives no privileged supervision beyond what outcome-reward baselines already use: the thinking and query rewards never see the gold answer (they assess intrinsic step quality), and the answer reward uses only the gold \emph{final} answer, exactly the supervision behind \SearchR{}'s EM reward. By contrast, StepSearch additionally requires gold \emph{intermediate} documents. Final evaluation is held-out EM against ground truth, which the judge cannot influence, and the largest gains appear on out-of-domain multi-hop benchmarks (Section~\ref{sec:experiments}), the signature of improved search-and-reason skills rather than absorbed judge knowledge.

\paragraph{Robustness to Reward Hacking.}
A learned reward invites the concern that the policy might produce reasoning that merely \emph{sounds} faithful and specific to the judge. Three design choices limit this risk. First, the terminal signal remains grounded: the answer reward is scored against the gold answer and final evaluation is held-out EM, so plausible-but-useless reasoning cannot earn terminal credit. Second, decomposition makes hacking harder rather than easier: thinking and query are scored on separate concrete rubrics (Appendix~\ref{app:prompts}), and the query is judged \emph{before} retrieval, so it cannot be gamed through retrieval nondeterminism. Third, the reason-then-score protocol reduces shallow pattern-matching by the judge. Empirically, reward hacking would manifest as the judge reward rising while held-out EM stagnates or falls; instead, EM improves on all seven benchmarks including five out-of-domain ones (Table~\ref{tab:main-results-7b}), and a policy that had merely learned to flatter the judge would not transfer to unseen multi-hop tasks. We note that reliable step-level judging is aided by our short-horizon setting ($T \leq 4$): the judge evaluates locally checkable properties of a single think-then-query action rather than predicting a distant outcome (Appendix~\ref{app:why-search}).

\subsection{Step-Level GRPO Optimization}\label{sec:step-grpo}

We now describe how the truncated step-level samples and dense rewards are integrated into the GRPO optimization framework.

\paragraph{Step-Level Policy Gradient.}
At each step $t$, given $k$ candidate actions $\{a_t^{(1)}, \ldots, a_t^{(k)}\}$ with step-level advantages $\{\hat{A}_t^{(1)}, \ldots, \hat{A}_t^{(k)}\}$ (Eq.~\ref{eq:step-advantage}), we compute the clipped policy gradient objective for each candidate. Following \SearchR{}, we apply loss masking to retrieved tokens: let $I(y_l) = 1$ if $y_l$ is generated by the LLM and $I(y_l) = 0$ if $y_l$ is a retrieved token. The step-level objective is:

\begin{equation}
\label{eq:masked-step-grpo}
\mathcal{J}_t^{(j)}(\theta) = \frac{1}{\sum_l I(y_l)} \sum_{\substack{l:\,I(y_l)=1}} \min\big( \rho_l \, \hat{A}_t^{(j)},
\text{clip}(\rho_l, 1{-}\epsilon, 1{+}\epsilon) \, \hat{A}_t^{(j)} \big),
\end{equation}
where $\rho_l = \pi_\theta(y_l \mid x, y_{<l}; \se) / \pi_{\theta_{\text{old}}}(y_l \mid x, y_{<l}; \se)$ is the per-token importance ratio and the summation runs only over LLM-generated tokens in action $a_t^{(j)}$.
\noindent
The complete \Ours{} training objective aggregates over all steps and all candidates:

\begin{equation}
\label{eq:full-objective}
\mathcal{J}_{\text{\Ours{}}}(\theta) = \E_{x \sim \mathcal{D}} \bigg[ \sum_{t=1}^{T} \frac{1}{k} \sum_{j=1}^{k} \mathcal{J}_t^{(j)}(\theta) \nonumber - \beta \, \mathbb{D}_{\text{KL}}\!\left[\pi_\theta \| \pi_{\text{ref}}\right] \bigg],
\end{equation}
where $\beta$ is the KL regularization coefficient and the KL divergence is computed only over LLM-generated tokens.

\section{Theoretical Analysis}\label{sec:theory}

The truncated sampling strategy introduced in Section~\ref{sec:truncated} is designed to reduce the variance of advantage estimates by isolating variation to a single decision point. But how much variance reduction does this actually provide, and under what conditions? We now formalize these intuitions, proving that truncated step-level sampling achieves up to a $T$-fold reduction in advantage variance compared to standard full-trajectory GRPO, directly yielding lower-variance policy gradients that enable faster and more stable training.

We formally analyze the variance reduction from truncated step-level sampling. Consider a $T$-step trajectory $\tau = (a_1, \ldots, a_T)$ with step-level rewards $r_t = r(a_t, \tau_{<t})$. To isolate the effect of the sampling strategy, we compare two estimators under the same additive reward $R(\tau) = \sum_t r_t$: \textbf{(A)}~standard GRPO, which samples $G$ complete trajectories and computes trajectory-level advantages $\hat{A}_i = R(\tau_i) - \frac{1}{G}\sum_l R(\tau_l)$; and \textbf{(B)}~our truncated method, which fixes a prefix $\tau_{<t}$ and samples $k$ actions at step $t$ with step-level advantages $\hat{A}_t^{(j)} = r_t^{(j)} - \frac{1}{k}\sum_l r_t^{(l)}$ (see Appendix~\ref{app:gradient-estimators} for the full gradient expressions).

\begin{theorem}[Variance Reduction via Truncated Sampling]\label{thm:var-reduction}
Let $\tau = (a_1, \ldots, a_T)$ be a $T$-step trajectory. Suppose the trajectory-level reward decomposes additively as $R(\tau) = \sum_{t=1}^T r_t(a_t, \tau_{<t})$, where $r_t$ is the step-$t$ reward. Assume the following conditions hold: 1) Non-negative future covariance: for each step $t$ and any fixed prefix $\tau_{<t}$, the covariance between the current step reward $r_t$ and the sum of future rewards $F_t$ satisfies $\Cov(r_t, F_t \mid \tau_{<t}) \geq 0$. 2) Conditional independence: Step rewards are conditionally independent given the prefix trajectory. 3) Variance symmetry: Step rewards have comparable variance across steps, i.e., $\E_{\tau_{<t}}[\Var[r_t \mid \tau_{<t}]] \approx \bar{v}$ for all $t$.

Then the per-sample variance of the scalar advantage in the truncated estimator satisfies:
\begin{equation}\label{eq:var-bound-general}
    \E_{\tau_{<t}}\!\left[\Var[\hat{A}_t^{(j)} \mid \tau_{<t}]\right] \leq \Var\!\left[\hat{A}_i\right],
\end{equation}
where the left side is the expected (over prefixes) per-sample variance in the truncated estimator and the right side is the per-sample variance in the full-trajectory estimator. This holds under Assumption~1 alone. Moreover, under all three assumptions with equal group sizes $k = G$:
\begin{equation}\label{eq:var-bound}
    \E_{\tau_{<t}}\!\left[\Var[\hat{A}_t^{(j)} \mid \tau_{<t}]\right] \leq \frac{1}{T} \cdot \Var\!\left[\hat{A}_i\right].
\end{equation}
\end{theorem}

\noindent\textit{Note:} The two estimators target different quantities, the trajectory-level advantage $\hat{A}_i$ estimates the deviation of the full return $R(\tau)$, while the step-level advantage $\hat{A}_t^{(j)}$ estimates the deviation of the single-step reward $r_t$. The comparison below therefore characterizes a bias-variance trade-off: the step-level estimator achieves lower variance at the cost of discarding future-reward information (see Remark~\ref{rem:bias-variance-estimator} for when this substitution is justified).

\begin{proof}[Proof Sketch]
The proof proceeds in two parts (full details in Appendix~\ref{app:proof-theorem}).

\textbf{Part~1 (General bound, Eq.~\ref{eq:var-bound-general}).}
Fixing the prefix $\tau_{<t}$ eliminates all randomness except the action at step~$t$. By the law of total variance, the expected conditional variance $\E_{\tau_{<t}}[\Var[R(\tau) \mid \tau_{<t}]]$ is at most $\Var[R(\tau)]$. Because past rewards are constant given $\tau_{<t}$, the trajectory reward decomposes as $R(\tau)\mid\tau_{<t}=c+r_t+F_t$, and Assumption~1 ($\Cov(r_t,F_t\mid\tau_{<t})\geq 0$) ensures $\Var[r_t\mid\tau_{<t}]\leq\Var[R(\tau)\mid\tau_{<t}]$. Combining these bounds with $k=G$ yields Eq.~\ref{eq:var-bound-general}.

\textbf{Part~2 ($T$-fold reduction, Eq.~\ref{eq:var-bound}).}
Under conditional independence (Assumption~2), the trajectory variance decomposes as $\Var[R(\tau)] \geq \sum_t \E_{\tau_{<t}}[\Var[r_t\mid\tau_{<t}]]$. Variance symmetry (Assumption~3) then gives $\E_{\tau_{<t}}[\Var[r_t\mid\tau_{<t}]] \leq \frac{1}{T}\Var[R(\tau)]$, which combined with Part~1 yields the $1/T$ factor. As a corollary, we show that truncated sampling also yields a $T$-fold reduction in total token generation cost to achieve the same advantage variance as standard GRPO (Proposition~\ref{cor:sample-eff} in Appendix~\ref{app:sample-efficiency}). 
\end{proof}

Since the policy gradient is a linear function of these advantages, lower advantage variance directly translates to lower-variance gradient estimates, enabling faster convergence and better final solutions (see Appendix~\ref{app:variance-discussion} for a detailed discussion). We further discuss the bias-variance trade-off of estimator substitution and the credit assignment benefits of dense rewards in Appendix~\ref{app:proof-remarks}.

\section{Experiments}\label{sec:experiments}

\paragraph{Datasets}
We evaluate \Ours{} on seven benchmark datasets spanning two categories:
(1)~\textbf{Single-Hop Question Answering}: NQ~\citep{kwiatkowski2019natural}, TriviaQA~\citep{joshi2017triviaqa}, and PopQA~\citep{mallen2022not}.
(2)~\textbf{Multi-Hop Question Answering}: HotpotQA~\citep{yang2018hotpotqa}, 2WikiMultiHopQA~\citep{ho2020constructing}, Musique~\citep{trivedi2022musique}, and Bamboogle~\citep{press2022measuring}.
These datasets encompass a diverse range of search-with-reasoning challenges, enabling comprehensive evaluation across both single-turn and multi-hop retrieval scenarios.

\paragraph{Baselines}
We compare \Ours{} against the following baselines: \textbf{Inference without Retrieval:} Direct generation and Chain-of-Thought (CoT) reasoning~\citep{wei2022chain}. \textbf{Inference with Retrieval:} RAG~\citep{lewis2020retrieval}, IRCoT~\citep{trivedi2022interleaving}, and Search-o1~\citep{li2025search}. \textbf{Fine-Tuning Methods:} Supervised fine-tuning (SFT)~\citep{chung2024scaling} and RL without search (R1)~\citep{guo2025deepseek}. \textbf{Search with RL:} \SearchR{}~\citep{jin2025searchr1trainingllmsreason} (sparse outcome reward), ZeroSearch~\citep{sun2025zerosearchincentivizesearchcapability}, ReSearch~\citep{chen2025researchlearningreasonsearch}, and StepSearch~\citep{stepsearch2025} (step-wise process rewards). We do not compare with deep research systems (e.g., OpenAI Deep Research, Gemini Deep Research), as these target long-form report generation with iterative web browsing and fundamentally different evaluation protocols (e.g., human preference ratings over multi-page reports), making them not directly comparable to factoid QA benchmarks evaluated via exact match. Full experimental setup details (hyperparameters, hardware, retrieval configuration) are provided in Appendix~\ref{app:setup}.

\subsection{Main Results}
The main results comparing \Ours{} with all baselines across seven datasets are presented in Table~\ref{tab:main-results-7b}. We make the following key observations:

\textbf{\Ours{} consistently performs best across all benchmarks.}
On the 7B model, \Ours{} obtains an average EM of 0.461, representing a 3.0\% absolute (7.0\% relative) improvement over \SearchR{} (0.431) and outperforming the best prior results on every individual dataset. On the 3B model, the improvement over \SearchR{} is even more substantial (0.396 vs.\ 0.303, a 30.7\% relative improvement), demonstrating that smaller models benefit more from the dense step-level supervision.

\textbf{Improvements are largest on hard multi-hop benchmarks.}
The gains of \Ours{} over prior methods are non-uniform and scale with task difficulty. On the harder out-of-domain multi-hop datasets, \Ours{} (7B) achieves the largest absolute improvements: on Musique, the gain over \SearchR{} is +5.1\% and over StepSearch +3.1\%; on Bamboogle, +6.2\% and +2.7\%, respectively. On 2WikiMultiHopQA, \Ours{} obtains 0.413 vs.\ StepSearch's 0.385 (+2.8\%) and \SearchR{}'s 0.382 (+3.1\%). This pattern is consistent with our hypothesis that dense step-level rewards help most when complex multi-step reasoning is required, as the credit assignment problem is most severe for longer trajectories. Notably, \Ours{} is the only method that consistently outperforms both \SearchR{} and StepSearch across all four multi-hop benchmarks, as prior methods show complementary strengths.

\textbf{Gains on single-hop QA and out-of-domain generalization.}
On the single-hop QA benchmarks, \Ours{} outperforms \SearchR{} by 1.3--1.9\% absolute EM. Notably, five of our seven evaluation benchmarks are out-of-domain (trained only on NQ+HotpotQA), and the largest gains occur on these unseen multi-hop tasks, demonstrating that step-level supervision teaches transferable reasoning skills rather than dataset-specific shortcuts.

\textbf{Smaller models benefit more from dense supervision.}
On the 3B model, gains over \SearchR{} are dramatically larger on multi-hop benchmarks (e.g., +16.7\% on Musique, +27.3\% on Bamboogle), suggesting smaller models benefit most from step-level supervision.

\begin{table}[t]
    \centering
    \small
    \setlength{\tabcolsep}{4pt}
    \renewcommand{\arraystretch}{1.15}
    \caption{Main results (Exact Match) on Qwen2.5-7B-Base and Qwen2.5-3B-Base across seven QA benchmarks. Best results are \textbf{bolded}, second best are \underline{underlined}. \SearchR{} and \Ours{} are trained on NQ+HotpotQA. StepSearch is trained on MuSiQue (19k) and does not report single-hop QA results. $^\dagger$Evaluated on Wiki-18 knowledge base.}
    \vspace{0.2cm}
    \label{tab:main-results-7b}
    \label{tab:main-results-3b}
    \begin{tabular}{l ccc cccc c}
        \toprule
        \textbf{Method} & \textbf{NQ} & \textbf{TriviaQA} & \textbf{PopQA} & \textbf{HotpotQA} & \textbf{2Wiki} & \textbf{Musique} & \textbf{Bamboogle} & \textbf{Avg.} \\
        \midrule
        \multicolumn{9}{l}{\textit{Qwen2.5-7B-Base}} \\
        \midrule
        CoT & 0.228 & 0.529 & 0.231 & 0.217 & 0.261 & 0.045 & 0.168 & 0.240 \\
        RAG & 0.371 & 0.582 & 0.339 & 0.287 & 0.231 & 0.061 & 0.214 & 0.298 \\
        Search-o1 & 0.321 & 0.548 & 0.298 & 0.193 & 0.181 & 0.053 & 0.302 & 0.271 \\
        ZeroSearch & 0.411 & 0.601 & 0.389 & 0.294 & 0.275 & 0.102 & 0.258 & 0.333 \\
        ReSearch & 0.398 & 0.594 & 0.372 & 0.294 & 0.264 & 0.143 & 0.373 & 0.348 \\
        \SearchR{} & \underline{0.480} & \underline{0.638} & \underline{0.457} & \underline{0.433} & 0.382 & 0.196 & 0.432 & \underline{0.431} \\
        StepSearch$^\dagger$ & -- & -- & -- & 0.380 & \underline{0.385} & \underline{0.216} & \underline{0.467} & --\\
        \textbf{\Ours{} (Ours)} & \textbf{0.497} & \textbf{0.652} & \textbf{0.470} & \textbf{0.451} & \textbf{0.413} & \textbf{0.247} & \textbf{0.494} & \textbf{0.461} \\
        \midrule
        \multicolumn{9}{l}{\textit{Qwen2.5-3B-Base}} \\
        \midrule
        CoT & 0.168 & 0.451 & 0.196 & 0.163 & 0.257 & 0.032 & 0.072 & 0.191 \\
        RAG & 0.312 & 0.504 & 0.296 & 0.251 & 0.221 & 0.051 & 0.076 & 0.244 \\
        Search-o1 & 0.267 & 0.465 & 0.254 & 0.240 & 0.207 & 0.045 & 0.316 & 0.256 \\
        ZeroSearch & 0.348 & 0.531 & 0.365 & 0.260 & 0.234 & 0.056 & 0.096 & 0.270 \\
        ReSearch & 0.339 & 0.518 & 0.342 & 0.261 & 0.228 & 0.074 & 0.184 & 0.278 \\
        \SearchR{} & \underline{0.406} & \underline{0.587} & \underline{0.435} & 0.284 & 0.273 & 0.049 & 0.088 & 0.303 \\
        StepSearch$^\dagger$ & -- & -- & -- & \underline{0.329} & \underline{0.339} & \underline{0.181} & \underline{0.328} & -- \\
        \textbf{\Ours{} (Ours)} & \textbf{0.425} & \textbf{0.603} & \textbf{0.451} & \textbf{0.351} & \textbf{0.368} & \textbf{0.216} & \textbf{0.361} & \textbf{0.396} \\
        \bottomrule
    \end{tabular}
\end{table}

\paragraph{Ablation Study}

Table~\ref{tab:ablation} ablates each component of \Ours{} on Qwen2.5-7B-Base across multi-hop benchmarks, where multi-step trajectories make them the most informative setting for isolating the effects of truncated sampling and step-level rewards.
Variant~(a) applies the \emph{identical} Gemma3-27B decomposed rewards with standard full-trajectory GRPO sampling (mirroring SWiRL~\citep{goldie2025synthetic}); the comparison between variant~(a) and full \Ours{} therefore holds the reward model fixed and changes \emph{only} the sampling strategy, isolating the contribution of truncated sampling from that of the stronger reward signal.
Variant~(a) already improves over \SearchR{} (+2.9\% avg.), but full \Ours{} achieves a further +1.1\% gain, with the largest improvements on the hardest benchmarks (Musique $+$1.1\%, Bamboogle $+$1.3\%), confirming that the gains do not come simply from a better reward model and that the sampling strategy matters beyond the reward signal (consistent with Theorem~\ref{thm:var-reduction}).
Removing LLM-judge rewards (variant~b) causes a larger drop of 2.4\% average EM, with the biggest impact on harder datasets (Musique $-$2.9\%, Bamboogle $-$3.7\%).
Variant~(c), with only EM reward, performs close to \SearchR{} (0.368 vs.\ 0.361), confirming that the synergy of both components yields the full improvement.

\newpage

\begin{wrapfigure}{r}{0.41\textwidth}
    \centering
    \vspace{-0.5cm}
    \includegraphics[width=\linewidth]{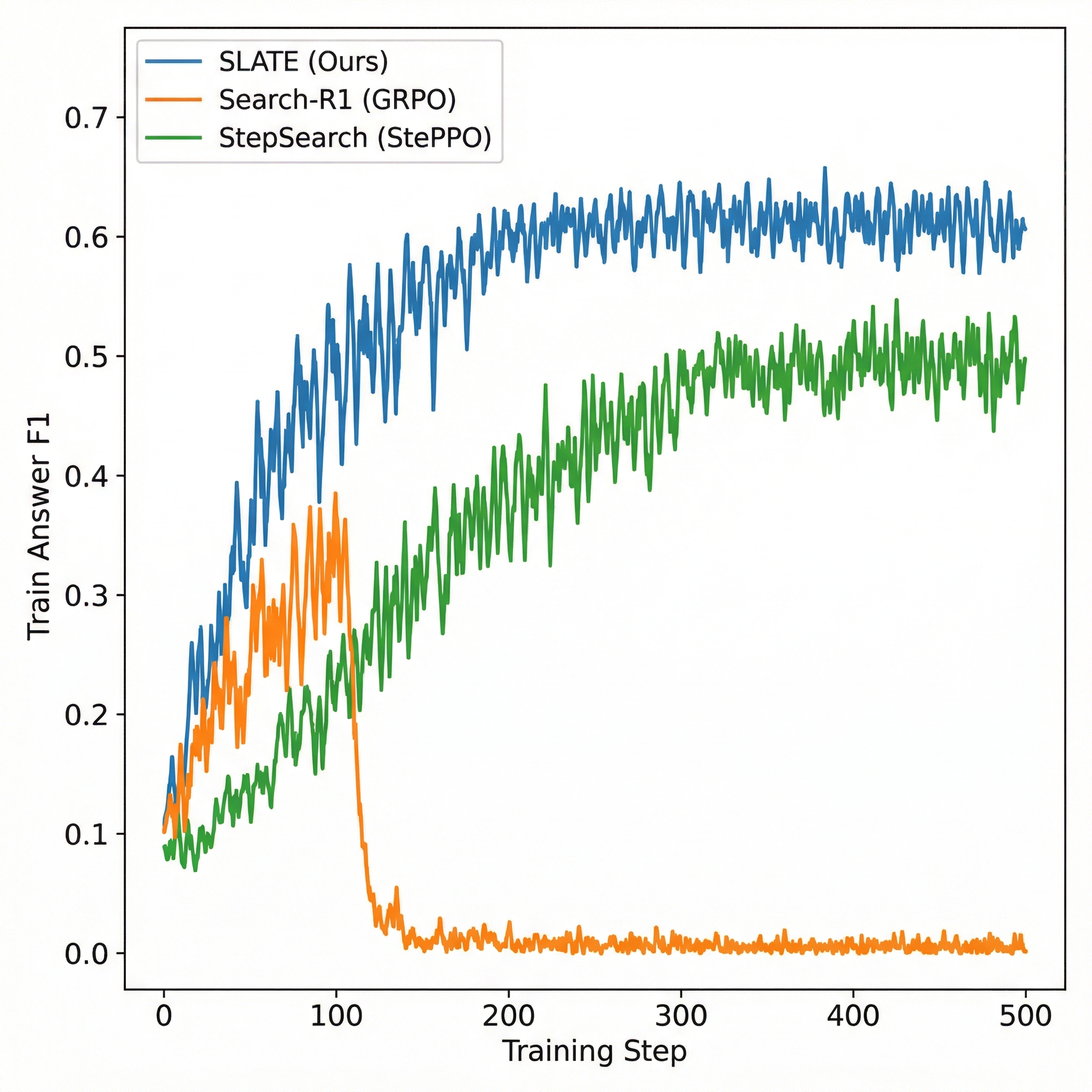}
    \vspace{-10px}
    \caption{Training dynamics on Qwen2.5-7B-Base. Each curve shows the method's \emph{own} training reward, so the curves indicate convergence speed and stability rather than directly comparable reward values.}
    \vspace{-0.3cm}
    \label{fig:training-dynamics}
\end{wrapfigure}

\paragraph{Training Dynamics}

We compare the training reward curves of \Ours{} against \SearchR{} (GRPO) and StepSearch (StePPO) in Figure~\ref{fig:training-dynamics}. Two caveats make this comparison precise. First, each method is plotted against its own training reward (\Ours{}'s dense judge reward, \SearchR{}'s EM reward, StepSearch's StePPO reward), so the curves share a step axis but not a reward axis; the informative comparison is convergence speed and stability, not absolute reward values. Second, the per-question update budget is comparable across methods: with $k = G$, \Ours{} applies gradients to $\sim\!kL$ policy tokens per question, the same order as GRPO's $GL$ (Section~\ref{sec:truncated}), so \Ours{} does not consume more generated data per question but reorganizes the same budget into prefix-isolated groups. With this in mind, \Ours{} exhibits three notable properties:
\textbf{(1) Faster convergence:} \Ours{} reaches its peak training reward approximately 20\% faster than StepSearch, attributable to the denser gradient signal from step-level rewards.
\textbf{(2) Higher reward ceiling:} \Ours{} plateaus at a consistently higher level on its own reward scale than both baselines do on theirs, reflecting the improved credit assignment from truncated sampling.
\textbf{(3) Greater stability:} Unlike GRPO which can exhibit reward collapse, \Ours{} maintains stable optimization throughout training due to the lower-variance advantage estimates predicted by Theorem~\ref{thm:var-reduction}. We also study the effect of the number of truncated samples $k$ per step in Appendix~\ref{app:group-size}. Performance improves steadily from $k=1$ to $k=5$ with diminishing returns at $k=7$, consistent with the $1/k$ variance reduction predicted by Theorem~\ref{thm:var-reduction}.

\begin{table}[t]
    \centering
    \small
    \setlength{\tabcolsep}{3.8pt}
    \renewcommand{\arraystretch}{1.15}
    \caption{Ablation study on Qwen2.5-7B-Base (Exact Match) on multi-hop QA benchmarks.}\label{tab:ablation}
    \vspace{0.2cm}
    \begin{tabular}{l cccc c}
        \toprule
        \textbf{Variant} & \textbf{HotpotQA} & \textbf{2Wiki} & \textbf{Musique} & \textbf{Bamboogle} & \textbf{Avg.} \\
        \midrule
        \SearchR{} (baseline) & 0.433 & 0.382 & 0.196 & 0.432 & 0.361 \\
        \midrule
        (a) GRPO + dense judge rewards (full-traj.) & 0.443 & 0.401 & 0.236 & 0.481 & 0.390 \\
        (b) w/o LLM-judge rewards & 0.440 & 0.393 & 0.218 & 0.457 & 0.377 \\
        (c) Truncated + EM reward only & 0.437 & 0.387 & 0.205 & 0.444 & 0.368 \\
        \midrule
        \textbf{\Ours{} (full)} & \textbf{0.451} & \textbf{0.413} & \textbf{0.247} & \textbf{0.494} & \textbf{0.401} \\
        \bottomrule
    \end{tabular}
\end{table}

\paragraph{Why Process Rewards Succeed in Search.}
In math, step-level rewards as RL training signals generally underperform outcome-based rewards~\citep{uesato2022solving,lightman2023let}. Retrieval-augmented reasoning differs in several key ways: steps are \emph{externally grounded} by retrieval results (eliminating capability mismatch), rewards are \emph{evaluative} rather than predictive, and trajectories are short ($T \leq 4$) and modular. Truncated sampling further addresses prefix confounding by isolating variation to a single action (Theorem~\ref{thm:var-reduction}). We provide a detailed discussion in Appendix~\ref{app:why-search}.

\section{Conclusions}

We presented \Ours{}, a training method for retrieval-augmented reasoning whose key insight is that \emph{how} step-level optimization is performed matters as much as \emph{what} reward signal is used. By sampling $k$ continuations from a shared prefix, our method isolates variation to a single action with provable variance guarantees (up to $T$-fold reduction), a formal contribution absent from prior step-level reward approaches~\citep{goldie2025synthetic,stepsearch2025}. Combined with decomposed ternary rewards, \Ours{} significantly outperforms both sparse-reward and process-reward baselines across seven benchmarks, with ablations confirming that truncated sampling yields gains above and beyond step-level rewards alone.


\bibliography{references}
\bibliographystyle{colm2026_conference}

\newpage

\appendix
\section{Appendix}

\subsection{The \Ours{} Algorithm}\label{app:algorithm}
The proposed \Ours{} model is presented in Algorithm~\ref{alg:dense-r1}.
\begin{algorithm}[]
\caption{Step-Level Sampling with Dense LLM-Judge Rewards}
\label{alg:dense-r1}
\small
\begin{algorithmic}[1]
\Require Policy $\pi_\theta$, search engine $\se$, LLM judge $\mathcal{R}$, dataset $\mathcal{D}$, group size $k$, temperature $\eta$, max steps $B$, termination bonus $\lambda$
\For{each question $x \sim \mathcal{D}$}
    \State $\tau_{<1} \gets \emptyset$, \; $t \gets 1$, \; $\Delta\theta \gets 0$
    \While{$t \leq B$}
        \For{$j = 1, \ldots, k$} \Comment{Step-level sampling}
            \State Sample $a_t^{(j)} = (s_t^{(j)}, q_t^{(j)}) \sim \pi_\theta(\cdot \mid x, \tau_{<t})$
            \State $r_{\text{think}}^{(j)} \gets \mathcal{R}^{\text{think}}(s_t^{(j)}, \tau_{<t})$
            \If{$a_t^{(j)}$ contains \texttt{<search>}}
                \State $r_t^{(j)} \gets r_{\text{think}}^{(j)} + \mathcal{R}^{\text{query}}(q_t^{(j)}, s_t^{(j)}, \tau_{<t})$
            \ElsIf{$a_t^{(j)}$ contains \texttt{<answer>}}
                \State $r_t^{(j)} \gets r_{\text{think}}^{(j)} + \mathcal{R}^{\text{ans}}(a^{(j)}, a_{\text{gold}}, \tau_{<t}) + \lambda \cdot \frac{B - t}{B}$
            \EndIf
        \EndFor
        \Statex \Comment{Compute step-level GRPO advantages}
        \State $\bar{r}_t \gets \frac{1}{k}\sum_j r_t^{(j)}$, \; $\sigma_t \gets \text{std}(\{r_t^{(j)}\})$
        \For{$j = 1, \ldots, k$}
            \State $\hat{A}_t^{(j)} \gets (r_t^{(j)} - \bar{r}_t) / (\sigma_t + \epsilon)$
        \EndFor
        \Statex \Comment{Update policy parameters}
        \State Accumulate gradient: $\Delta\theta \mathrel{+}= \nabla_\theta \Big[\frac{1}{k}\sum_{j=1}^{k} \mathcal{J}_t^{(j)}(\theta) - \beta \, \mathbb{D}_{\text{KL}}[\pi_\theta \| \pi_{\text{ref}}]\Big]$
        \Statex \hspace{2em} $\mathcal{J}_t^{(j)}(\theta) \!=\! \frac{1}{\sum_l I(y_l)} \!\!\sum_{l:\,I(y_l)=1}\!\! \min\!\big(\rho_l \hat{A}_t^{(j)},\; \text{clip}(\rho_l, 1{-}\epsilon, 1{+}\epsilon)\hat{A}_t^{(j)}\big)$
        \Statex \hspace{2em} $\rho_l = \pi_\theta(y_l \mid y_{<l}) / \pi_{\theta_{\text{old}}}(y_l \mid y_{<l})$, \; $I(y_l) = \mathbbm{1}[y_l \text{ is LLM-generated}]$
        \Statex \Comment{Extend trajectory via reward-weighted sampling}
        \State $j^* \sim \text{softmax}\!\big(\hat{A}_t^{(1)}, \ldots, \hat{A}_t^{(k)} \,/\, \eta\big)$
        \State $\tau_{<t+1} \gets \tau_{<t} \cup \{a_t^{(j^*)}, \se(q_t^{(j^*)})\}$
        \State $t \gets t + 1$
        \If{$a_t^{(j^*)}$ contains \texttt{<answer>}}
            \State \textbf{break}
        \EndIf
    \EndWhile
    \State $\theta \gets \theta + \alpha \, \Delta\theta$ \Comment{Batched update after trajectory completes}
\EndFor
\end{algorithmic}
\vspace{0.3em}
\noindent\textit{Note:} Gradients are accumulated across all steps within a trajectory and applied in a single batched update after the trajectory completes. In practice, updates are further batched across multiple examples in the training batch.
\end{algorithm}

\subsection{Experimental Setup}\label{app:setup}

We conduct experiments using Qwen2.5-7B-Base and Qwen2.5-3B-Base~\citep{yang2024qwen2}.
For retrieval, we use the 2018 Wikipedia dump~\citep{karpukhin2020dense} as the knowledge source and E5~\citep{wang2022text} as the retriever, retrieving the top 3 passages per query.
Following \SearchR{}, we merge the training sets of NQ and HotpotQA to form the unified training dataset.
Exact Match (EM) is used as the primary evaluation metric~\citep{yu2024rankrag}.
We use GRPO as the base RL algorithm with a policy learning rate of $1 \times 10^{-6}$, $k = 5$ truncated samples per step, clip ratio $\epsilon = 0.2$, KL coefficient $\beta = 0.001$, and early-termination bonus $\lambda = 0.1$.
For trajectory construction we use reward-weighted sampling (Section~\ref{sec:truncated}) with temperature $\eta = 0.7$ to select which action extends the prefix at each step.
The LLM-as-judge reward model uses Gemma3-27B as the evaluator.
Training is performed for 500 steps on two NVIDIA A100 GPUs using LoRA~\citep{hu2022lora} (rank 16, $\alpha = 64$) for parameter-efficient fine-tuning in bfloat16 precision, with a batch size of 32, maximum sequence length of 4096 tokens, and maximum action budget $B = 4$.
Retrieved token loss masking is applied following \SearchR{}.

\subsection{Gradient Estimators for Full and Truncated Step-level Trajectories}\label{app:gradient-estimators}

The two gradient estimation strategies compared in Section~\ref{sec:theory} are defined as follows. For \textbf{full-trajectory sampling (GRPO)}, the gradient estimate for step $t$ in trajectory $i$ is:
\begin{equation}\label{eq:grpo-grad}
    \hat{g}_t^{\text{GRPO}} = \frac{1}{G}\sum_{i=1}^{G} \hat{A}_i \cdot \nabla_\theta \log \pi_\theta(a_{t,i} \mid \tau_{<t,i}),
\end{equation}
where $\hat{A}_i = R(\tau_i) - \frac{1}{G}\sum_l R(\tau_l)$. For \textbf{truncated step-level sampling (\Ours{})}, the gradient at step $t$ is:
\begin{equation}\label{eq:ours-grad}
    \hat{g}_t^{\text{Ours}} = \frac{1}{k}\sum_{j=1}^{k} \hat{A}_t^{(j)} \cdot \nabla_\theta \log \pi_\theta(a_t^{(j)} \mid \tau_{<t}),
\end{equation}
where $\hat{A}_t^{(j)} = r_t^{(j)} - \frac{1}{k}\sum_l r_t^{(l)}$.

\subsection{A Proof for Theorem~\ref{thm:var-reduction}: Variance Reduction via Truncated Sampling}\label{app:proof-theorem}
\begin{proof}
We decompose the proof into two parts. Part~1 shows that truncated sampling never increases variance in expectation over prefixes (using Assumption~1), and Part~2 quantifies the improvement as a $T$-fold reduction under all three assumptions.

\paragraph{Part 1: General Variance Bound.}

The core intuition is that fixing the prefix $\tau_{<t}$ eliminates all sources of randomness except the action at step $t$, reducing the variance of the advantage estimate on average.

Consider the full-trajectory advantage $\hat{A}_i = R(\tau_i) - \bar{R}$, where $\bar{R} = \frac{1}{G}\sum_i R(\tau_i)$. The variance of $\hat{A}_i$ taken over random trajectories $\tau_i \sim \pi_\theta(\cdot \mid x; \se)$ depends on the total variability of $R(\tau)$:
\begin{align}
    \Var[\hat{A}_i] &= \Var[R(\tau_i) - \bar{R}] \nonumber = \left(1 - \tfrac{1}{G}\right) \Var[R(\tau)].
\end{align}
\noindent
By the law of total variance applied to the trajectory prefix $\tau_{<t}$:
\begin{align}\label{eq:total-var}
    \Var[R(\tau)] &= \underbrace{\E_{\tau_{<t}}\!\left[\Var[R(\tau) \mid \tau_{<t}]\right]}_{\text{within-prefix variance}} \nonumber\quad+ \underbrace{\Var_{\tau_{<t}}\!\left[\E[R(\tau) \mid \tau_{<t}]\right]}_{\text{between-prefix variance} \;\geq\; 0}.
\end{align}
\noindent
Since both terms on the right are non-negative, the expected conditional variance is bounded by the unconditional variance:
\begin{equation}\label{eq:cond-var-bound}
    \E_{\tau_{<t}}\!\left[\Var[R(\tau) \mid \tau_{<t}]\right] \leq \Var[R(\tau)].
\end{equation}
\noindent
Now, in our method, the step-level advantage $\hat{A}_t^{(j)}$ is computed with the prefix $\tau_{<t}$ fixed. Its conditional variance is:
\begin{equation}
    \Var[\hat{A}_t^{(j)} \mid \tau_{<t}] = \left(1 - \frac{1}{k}\right) \Var[r_t \mid \tau_{<t}].
\end{equation}
\noindent
We now show that $\Var[r_t \mid \tau_{<t}] \leq \Var[R(\tau) \mid \tau_{<t}]$ for each prefix. Given a fixed prefix $\tau_{<t}$, the rewards from steps $1, \ldots, t{-}1$ are constants, so $R(\tau) \mid \tau_{<t} = c + r_t + F_t$ where $c = \sum_{t'<t} r_{t'}$ is constant and $F_t = \sum_{t'=t+1}^T r_{t'}$ denotes the future rewards. Therefore:
\begin{align}
    \Var[R(\tau) \mid \tau_{<t}] &= \Var[r_t + F_t \mid \tau_{<t}] \nonumber\\
    &= \Var[r_t \mid \tau_{<t}] + \Var[F_t \mid \tau_{<t}] + 2\Cov(r_t, F_t \mid \tau_{<t}).
\end{align}
Under Assumption~1 ($\Cov(r_t, F_t \mid \tau_{<t}) \geq 0$) and since $\Var[F_t \mid \tau_{<t}] \geq 0$, we obtain:
\begin{equation}
    \Var[r_t \mid \tau_{<t}] \leq \Var[R(\tau) \mid \tau_{<t}].
\end{equation}
Taking expectations over prefixes and applying Eq.~\ref{eq:cond-var-bound}:
\begin{equation}
    \E_{\tau_{<t}}\!\left[\Var[r_t \mid \tau_{<t}]\right] \leq \E_{\tau_{<t}}\!\left[\Var[R(\tau) \mid \tau_{<t}]\right] \leq \Var[R(\tau)].
\end{equation}
Combining these and setting $k = G$:
\begin{align}
    \E_{\tau_{<t}}\!\big[\Var[\hat{A}_t^{(j)} \mid \tau_{<t}]\big]
    &= \left(1 - \tfrac{1}{k}\right) \E_{\tau_{<t}}\!\big[\Var[r_t \mid \tau_{<t}]\big] \nonumber\\
    &\leq \left(1 - \tfrac{1}{G}\right) \Var[R(\tau)] = \Var[\hat{A}_i].
\end{align}
This establishes that, with equal group sizes, the truncated estimator has no more variance in expectation than the full-trajectory estimator. 

\paragraph{Part 2: $T$-fold Reduction Under Independence and Symmetry.}

Part~1 shows the truncated estimator is never worse; we now show it can be $T$ times \emph{better}. The intuition is that the total trajectory variance $\Var[R(\tau)]$ is the sum of variances from all $T$ steps, but the truncated estimator only ``sees'' the variance from one step, hence the $1/T$ factor.

Under Assumption~2 (conditional independence), for each step $t$, the reward $r_t$ depends only on $a_t$ and $\tau_{<t}$, and given the prefix $\tau_{<t}$, the step reward $r_t$ is independent of rewards at other steps conditioned on their respective prefixes. We establish the desired inequality by applying the law of total variance recursively. For any step $t$, the law of total variance gives:
\begin{equation}
    \Var[R(\tau)] = \E_{\tau_{<t}}\!\left[\Var[R(\tau) \mid \tau_{<t}]\right] + \Var_{\tau_{<t}}\!\left[\E[R(\tau) \mid \tau_{<t}]\right] \geq \E_{\tau_{<t}}\!\left[\Var[R(\tau) \mid \tau_{<t}]\right].
\end{equation}
From Part~1 (using Assumption~1), we already have $\Var[r_t \mid \tau_{<t}] \leq \Var[R(\tau) \mid \tau_{<t}]$ for each prefix $\tau_{<t}$. Summing over all steps:
\begin{equation}
    \sum_{t=1}^T \E_{\tau_{<t}}\!\left[\Var[r_t \mid \tau_{<t}]\right] \leq \sum_{t=1}^T \E_{\tau_{<t}}\!\left[\Var[R(\tau) \mid \tau_{<t}]\right].
\end{equation}
Under Assumption~2, the per-step conditional variances account for distinct, non-overlapping sources of randomness (each step's action is the sole source of variation given its prefix). Therefore:
\begin{equation}
    \Var[R(\tau)] \geq \sum_{t=1}^T \E_{\tau_{<t}}\!\left[\Var[r_t \mid \tau_{<t}]\right].
\end{equation}
\noindent
Under Assumption~3 (variance symmetry), $\E_{\tau_{<t}}[\Var[r_t \mid \tau_{<t}]] \approx \bar{v}$ for all $t$, so the right-hand side simplifies to $T \cdot \bar{v}$, giving:
\begin{equation}
    \Var[R(\tau)] \geq T \cdot \bar{v} \geq T \cdot \E_{\tau_{<t}}\!\left[\Var[r_t \mid \tau_{<t}]\right].
\end{equation}
Rearranging yields $\E_{\tau_{<t}}[\Var[r_t \mid \tau_{<t}]] \leq \frac{1}{T} \Var[R(\tau)]$. Combining with the result from Part~1 (with $k = G$):
\begin{align}
    \E_{\tau_{<t}}\!\big[\Var[\hat{A}_t^{(j)} \mid \tau_{<t}]\big]
    &= \left(1 - \tfrac{1}{G}\right)\E_{\tau_{<t}}\!\big[\Var[r_t \mid \tau_{<t}]\big] \nonumber\\
    &\leq \left(1 - \tfrac{1}{G}\right) \frac{1}{T}\Var[R(\tau)] \nonumber\\
    &= \frac{1}{T} \cdot \Var[\hat{A}_i]. 
\end{align}
\end{proof}

\subsection{Theoretical Remarks}\label{app:proof-remarks}

\begin{remark}[Bias-Variance Trade-off in the Estimator Comparison]\label{rem:bias-variance-estimator}
Theorem~\ref{thm:var-reduction} compares the variance of two \emph{different} advantage estimators: the trajectory-level $\hat{A}_i$ (based on $R(\tau)$) and the step-level $\hat{A}_t^{(j)}$ (based on $r_t$). Because the truncated estimator targets the step-level reward rather than the full return, it is not an unbiased estimator of the trajectory-level advantage. The practical validity of this substitution rests on the step-level reward being a low-bias proxy for the contribution of action $a_t$ to the trajectory outcome, a condition favored by our setting's short horizons ($T \leq 4$), externally grounded evaluation (retrieved documents are directly observable), and evaluative (rather than predictive) reward design (Section~\ref{sec:dense-reward}). We discuss this trade-off further in Appendix~\ref{app:bias-variance}.
\end{remark}

\begin{remark}[Credit Assignment]\label{rem:credit}
The \emph{reward design} provides an orthogonal benefit. By the data processing inequality, a binary EM reward provides at most 1 bit about the joint outcome of all $T$ actions, severely diluting the signal for any individual step. In contrast, the step-level LLM-judge reward $r_t$ directly evaluates $a_t$ in context, providing a substantially richer signal. Our ablation (Table~\ref{tab:ablation}) confirms that removing dense rewards causes larger drops than removing truncated sampling. We discuss the bias-variance trade-off of LLM-judge rewards in Appendix~\ref{app:bias-variance}.
\end{remark}

\subsection{Sample Efficiency}\label{app:sample-efficiency}

\begin{proposition}[Sample Efficiency]
\label{cor:sample-eff}
To achieve the same advantage variance as standard GRPO with $G$ full-trajectory samples, the truncated step-level method requires only $G/T$ samples per step under the conditions of Theorem~\ref{thm:var-reduction}, yielding a $T$-fold reduction in total token generation cost.
\end{proposition}

\subsection{Proof of Proposition~\ref{cor:sample-eff}}\label{app:proof-proposition}
\begin{proof}
The variance of each method's advantage estimator for step $t$ scales inversely with the number of samples. For standard GRPO, the variance of the mean advantage estimator is $\frac{1}{G} \Var[\hat{A}_i]$. For the truncated method, the expected conditional variance of the mean advantage estimator is:
\begin{equation}
    \frac{1}{k} \E_{\tau_{<t}}\!\left[\Var[\hat{A}_t^{(j)} \mid \tau_{<t}]\right] \leq \frac{1}{k} \cdot \frac{1}{T} \cdot \Var[\hat{A}_i],
\end{equation}
where the inequality follows from Theorem~\ref{thm:var-reduction}. Equating the two to find the minimum $k$:
\begin{equation}
    \frac{1}{k} \cdot \frac{1}{T} \cdot \Var[\hat{A}_i] = \frac{1}{G} \cdot \Var[\hat{A}_i] \quad \Longrightarrow \quad k = \frac{G}{T}.
\end{equation}
Thus, only $G/T$ samples per step suffice.

We now count total tokens generated. In standard GRPO, we sample $G$ complete trajectories of average length $L$, costing $G \cdot L$ tokens. In our method, each truncated sample generates only one step's worth of tokens, approximately $L/T$ tokens, since the full trajectory's $L$ tokens are spread over $T$ steps, and we repeat this sampling at each of the $T$ steps. The total cost is therefore:
\begin{equation}
    \underbrace{\frac{G}{T}}_{\text{samples per step}} \times \underbrace{\frac{L}{T}}_{\text{tokens per sample}} \times \underbrace{T}_{\text{steps}} = \frac{G \cdot L}{T}.
\end{equation}
The two sources of savings, $T\times$ fewer samples needed (due to lower per-sample variance) and $T\times$ fewer tokens per sample (due to truncation), together yield a $T^2$ reduction; paying back one factor of $T$ for repeating across all $T$ steps gives a net $T$-fold improvement over the $G \cdot L$ tokens required by standard GRPO. 
\end{proof}

\subsection{Variance Reduction and Convergence}\label{app:variance-discussion}

Theorem~\ref{thm:var-reduction} establishes that the truncated estimator yields lower-variance scalar advantages $\hat{A}_t^{(j)}$. Since the policy gradient $\hat{g}_t = \frac{1}{k}\sum_j \hat{A}_t^{(j)} \nabla_\theta \log \pi_\theta(a_t^{(j)} \mid \tau_{<t})$ is a linear function of these advantages, lower advantage variance directly translates to lower variance in the gradient estimates (modulated by the score function magnitudes).

Intuitively, the policy gradient estimate $\hat{g}$ acts as a noisy compass: it points toward the true gradient on average, but individual estimates may deviate substantially. In standard GRPO, the gradient signal for step~$t$ is weighted by the trajectory-level advantage $\hat{A}_i$, which conflates the quality of all $T$ actions. A trajectory may succeed despite a poor intermediate step, causing that step to be incorrectly reinforced; conversely, a trajectory may fail despite strong intermediate reasoning, penalizing all steps indiscriminately. These ``mislabeled'' updates are the practical manifestation of high advantage variance. Over many updates they cancel in expectation, but each wasted update consumes compute budget and slows progress. High variance also forces the use of smaller learning rates to avoid divergence, further limiting the speed of convergence. In our truncated method, fixing the prefix $\tau_{<t}$ and varying only step $t$ ensures that the advantage $\hat{A}_t^{(j)}$ reflects solely the quality of the current action, so every gradient update sends an accurate signal. Beyond faster convergence, lower variance can also lead to better \emph{final} solutions: cleaner gradients allow the optimizer to reliably descend into sharper, higher-performing regions of the loss landscape that noisy updates would overshoot or bounce out of, and they ensure that more of the finite training budget contributes useful learning signal rather than noise.

\subsection{Bias-Variance Trade-off}\label{app:bias-variance}

\begin{remark}[Bias-Variance Trade-off]
Our step-level LLM-judge rewards may introduce bias if the judge does not perfectly predict the contribution of step $t$ to the final outcome. However, this bias is typically small for well-calibrated LLM judges, and the significant variance reduction (Theorem~\ref{thm:var-reduction}) more than compensates, leading to faster convergence and better final performance. This aligns with the classical bias-variance trade-off in policy gradient methods, where moderate bias with substantially reduced variance yields better optimization dynamics~\citep{schulman2015high}.
\end{remark}

\paragraph{Why the Bias is Small in Our Setting.}
The truncated estimator deliberately targets the single-step reward $r_t$ instead of the full return $R(\tau)$, trading a small bias for an up-to-$T$-fold variance reduction. Three structural properties of retrieval-augmented reasoning keep this bias small (see also Appendix~\ref{app:why-search}): (1)~the rewards are \emph{evaluative} rather than predictive, scoring intrinsic step quality (e.g., a query's relevance, specificity, and searchability) rather than estimating $P(\text{correct answer} \mid \text{step})$, and a query that retrieves relevant documents is genuinely good regardless of downstream variation; (2)~each step is \emph{externally grounded} by the search engine, so a locally good action rarely leads to a dead end the policy cannot recover from; and (3)~horizons are short ($T \leq 4$), so the gap between local step value and trajectory outcome is inherently small.

\paragraph{Mitigations Built into the Method.}
Three components of \Ours{} further reduce the impact of the residual bias: (a)~reward-weighted sampling ($\eta = 0.7$) rather than pure best-of-$k$ preserves exploration and prevents the trajectory from collapsing onto a locally greedy but globally poor action; (b)~the early-termination bonus aligns the local per-step incentive with the global objective of answering efficiently; and (c)~future-reward information is not discarded across the trajectory: the selected action extends the prefix, and later steps receive their own step-level gradients, so the policy is optimized at every step of the trajectory. What \Ours{} does \emph{not} do is propagate retroactive credit from later rewards back to earlier prefix actions within the same trajectory; each action's gradient is computed against its own step-level group when it is sampled. This is a deliberate trade: retroactive credit is exactly the high-variance signal that Theorem~\ref{thm:var-reduction} removes. The approximation can fail when locally promising actions systematically lead to global dead ends, i.e., long horizons with tightly coupled steps, or deceptive or sparse intermediate rewards; we state this in the Limitations (Appendix~\ref{app:limitations}).

\subsection{Limitations}\label{app:limitations}

Our truncated sampling constructs a single trajectory by extending one selected action per step, which limits exploration compared to maintaining multiple full trajectories, and it does not propagate retroactive credit from later rewards back to earlier prefix actions within a trajectory (Appendix~\ref{app:bias-variance}). While the short horizons ($T \leq 4$) and externally grounded rewards in our setting mitigate this, the approach may be less effective in domains with longer horizons, tightly coupled steps, or deceptive intermediate rewards, where locally promising actions systematically lead to global dead ends. Additionally, our method relies on a large LLM judge (Gemma3-27B) for dense supervision: although policy-side generation is no more expensive than standard GRPO (Section~\ref{sec:truncated}), the up-to-$k$ judge calls per step introduce real training-time overhead and make the overall improvement dependent on the quality of the judge model. Finally, our evaluation follows the seven-benchmark, EM-based factoid QA protocol shared by \SearchR{}, StepSearch, ZeroSearch, and ReSearch to keep comparisons apples-to-apples; extending \Ours{} to tasks with richer tool use beyond a single search engine, and to tasks whose answers are not short string-matchable facts, is an important direction for future work.

\subsection{Early-Termination Bonus Details}\label{app:early-termination}

The bonus term $\lambda \cdot (B - t) / B$ in Eq.~\ref{eq:final-reward} encourages the model to produce an answer as soon as it has gathered sufficient information. Without such a term, the model may learn to issue superfluous search queries, each receiving a neutral or mildly positive reward from the LLM judge, even when the information needed to answer has already been retrieved. The bonus is largest when the model answers early (e.g., $\lambda \cdot \frac{3}{4}$ at step $t{=}1$ with $B{=}4$) and zero when it exhausts the full budget ($t{=}B$), creating a progressive incentive to terminate sooner. Crucially, this bonus only applies to answer candidates, so at any step where some of the $k$ sampled actions produce an answer and others produce a search query, the answer candidates receive a higher reward, creating a meaningful advantage signal that survives the group normalization in Eq.~\ref{eq:step-advantage}. This ensures the policy gradient directly reinforces early termination when additional search steps are unlikely to improve the answer.

\subsection{Rollout Volume and Response-Length Bias}\label{app:length-bias}

Compared to \SearchR{}, \Ours{} shifts training-time generation toward many \emph{short} samples: $k$ single-step continuations per step rather than $G$ full trajectories. Two properties of this shift merit explicit analysis. On volume, the short samples sum to a comparable token budget: per-question generation is $\sim\!kL$ tokens, the same order as GRPO's $GL$ at equal group sizes (Section~\ref{sec:truncated}), and Proposition~\ref{cor:sample-eff} shows the truncated scheme is in fact more token-efficient at matched advantage variance. On length bias, nothing in the reward design encourages brevity except the small early-termination bonus ($\lambda = 0.1$), which is deliberately the \emph{only} length-related pressure and applies exclusively to answer candidates (Appendix~\ref{app:early-termination}); the thinking and query rewards score quality, not length. If dense short-rollout supervision induced a harmful bias toward terse, under-searched trajectories, performance would degrade most on multi-hop benchmarks, which require \emph{more} retrieval steps; instead, \Ours{}'s largest gains are exactly there (e.g., +5.1\% on Musique and +6.2\% on Bamboogle over \SearchR{}, Table~\ref{tab:main-results-7b}). This indicates the method produces efficient rather than degenerate trajectories, consistent with the finding of \citet{wu2025hiprag} that dense process rewards reduce over-search without sacrificing accuracy.

\subsection{Why Process Rewards Succeed in Search but Not Math}\label{app:why-search}

While process reward models have proven effective as verifiers for mathematical reasoning~\citep{lightman2023let}, using step-level rewards as RL training signals has generally underperformed outcome-based verifiable rewards for policy optimization in math~\citep{uesato2022solving}. We argue that retrieval-augmented reasoning possesses structural properties that make process rewards fundamentally more reliable as RL training signals, and that our truncated sampling design addresses the remaining risks.

\paragraph{External Grounding Eliminates Capability Mismatch.}
In math, a process reward model (PRM) may assign high reward to an elegant proof strategy that the actor model cannot complete, for example, a sophisticated algebraic manipulation that only a larger model could carry out. This ``capability mismatch'' means the PRM rewards steps that look locally promising but lead to dead ends. In search, the ``hard part'' of each step is outsourced to the search engine: the model only needs to formulate a well-targeted query, and the retrieval system handles the actual information lookup. The capability ceiling for writing a good query is much lower than for completing an advanced proof, so the process reward cannot favor steps that exceed the actor's abilities.

\paragraph{Evaluative vs.\ Predictive Rewards.}
Math PRMs are inherently \emph{predictive}: they estimate $P(\text{correct final answer} \mid \text{this step})$, which requires anticipating all future reasoning. This is precisely where local minima arise, a step that appears promising may lead somewhere the model cannot follow. Our LLM-judge rewards are \emph{evaluative}: they assess intrinsic step quality along concrete dimensions (relevance, specificity, searchability for queries; clarity, progress, faithfulness for reasoning). A query that retrieves relevant documents is genuinely good regardless of downstream trajectory variation. The decomposed ternary design (Section~\ref{sec:dense-reward}) ensures each reward dimension is locally verifiable without needing to predict the trajectory's eventual outcome.

\paragraph{Short Horizons and Natural Decomposability.}
Search trajectories in our setting have $T \leq 4$ steps, each consisting of a discrete think-then-query action with a natural boundary (the search engine call). Math proofs can involve dozens of tightly coupled steps where the value of step 5 depends critically on steps 15--20. With shorter, more modular trajectories, the gap between local step quality and trajectory-level outcome is inherently smaller, reducing the risk that process rewards mislead the policy.

\paragraph{Truncated Sampling Prevents Local Minima.}
Even granting that search rewards are more locally informative, standard full-trajectory sampling still suffers from prefix confounding: a locally good step on a bad trajectory gets penalized, and a mediocre step on a lucky trajectory gets rewarded. This is exactly how local minima arise in PRM-based RL, the policy gets ``trapped'' by rewarding steps that merely co-occur with good outcomes. Our truncated sampling eliminates this confound: all $k$ candidates share the same prefix $\tau_{<t}$, so advantages reflect \emph{only} the current action's quality. The reward-weighted sampling strategy (temperature $\eta = 0.7$) further prevents greedy collapse into local optima during trajectory construction. The ablation in Table~\ref{tab:ablation} provides direct evidence: variant~(a), which uses LLM-judge rewards \emph{without} truncated sampling, underperforms full \Ours{} most on the hardest benchmarks, precisely where prefix-level confounds are most harmful.

\subsection{Hyperparameter Sensitivity}\label{app:group-size}

\Ours{} introduces three hyperparameters beyond standard GRPO: the group size $k$, the selection temperature $\eta$, and the early-termination bonus $\lambda$.

\paragraph{Group Size $k$.}
The group size is the hyperparameter that most directly affects both the variance of advantage estimates and the compute budget, and hence the one most likely to affect training stability. We study $k \in \{1, 3, 5, 7\}$ per step on Qwen2.5-7B-Base (Table~\ref{tab:group-size}). When $k=1$, the method reduces to standard REINFORCE with LLM-judge rewards (no group-relative advantage). Performance improves steadily from $k=1$ to $k=5$, with diminishing returns at $k=7$. This is consistent with our theoretical analysis: increasing $k$ reduces the variance of the step-level advantage estimate (Eq.~\ref{eq:step-advantage}), but the marginal benefit decreases as $1/k$. Reassuringly for stability, the trend is smooth and monotone-then-flat rather than brittle.

\paragraph{Selection Temperature $\eta$ and Termination Bonus $\lambda$.}
The remaining two hyperparameters play more circumscribed roles. The temperature $\eta$ affects only which candidate extends the prefix (Section~\ref{sec:truncated}), not the gradient computation itself: it interpolates between greedy best-of-$k$ ($\eta \to 0$), which risks collapsing exploration onto locally optimal actions, and uniform selection ($\eta \to \infty$), which ignores the reward signal; we use an intermediate $\eta = 0.7$ throughout. The bonus $\lambda = 0.1$ is deliberately small relative to the ternary reward range and is the only length-related pressure in the reward; its role and design are analyzed in Appendix~\ref{app:early-termination}, and its effect on response length in Appendix~\ref{app:length-bias}.

\begin{table}[h]
    \centering
    \small
    \setlength{\tabcolsep}{4pt}
    \renewcommand{\arraystretch}{1.15}
    \caption{Effect of group size $k$ (truncated samples per step) on Qwen2.5-7B-Base (EM).}
    \vspace{0.2cm}
    \label{tab:group-size}
    \begin{tabular}{l cccc c}
        \toprule
        \textbf{Group Size $k$} & \textbf{HotpotQA} & \textbf{2Wiki} & \textbf{Musique} & \textbf{Bamboogle} & \textbf{Avg.} \\
        \midrule
        $k = 1$ & 0.437 & 0.390 & 0.213 & 0.452 & 0.373 \\
        $k = 3$ & 0.446 & 0.405 & 0.237 & 0.483 & 0.393 \\
        $k = 5$ (default) & \textbf{0.451} & \textbf{0.413} & \textbf{0.247} & \textbf{0.494} & \textbf{0.401} \\
        $k = 7$ & 0.449 & 0.411 & 0.244 & 0.491 & 0.399 \\
        \bottomrule
    \end{tabular}
\end{table}

\subsection{LLM-as-Judge Reward Prompts}\label{app:prompts}

We provide the exact prompts used for the three LLM-as-judge reward components described in Section~\ref{sec:dense-reward}. In each prompt, the placeholders in curly braces are filled with the corresponding trajectory content at evaluation time. The judge is instructed to produce a chain-of-thought explanation before the score, enclosed in XML-style tags.

\paragraph{Thinking Reward Prompt.}
\begin{promptbox}
Evaluate the quality of the following reasoning step in a search-based question answering system.\newline
\newline
Context: \{context\}\newline
\newline
Current Thinking Step: \{thinking\}\newline
\newline
The reasoning should be based on the previous context and the question, nothing else.\newline
\newline
Evaluate this thinking step on these criteria:\newline
1. Relevance: Does it address the question appropriately?\newline
2. Clarity: Is the reasoning clear and logical?\newline
3. Specificity: Does it identify concrete information needs?\newline
4. Progress: Does it move toward answering the question?\newline
5. Faithfulness: Does it accurately reflect the information in the previous context? Is there any out-of-context information?\newline
\newline
Provide a score using EXACTLY one of these three values:\newline
- +1: GOOD -- Clear, relevant reasoning that identifies specific information needs and moves toward answering the question\newline
- ~0: ACCEPTABLE -- Reasoning is somewhat relevant but vague, lacks specificity, or makes only minimal progress\newline
- -1: BAD -- Irrelevant, misleading, or counterproductive reasoning that does not help answer the question\newline
\newline
First provide your reasoning, then the score. Use this exact format:\newline
<explanation> Your reasoning here </explanation>\newline
<score> numerical score </score>
\end{promptbox}

\paragraph{Query Generation Reward Prompt.}
\begin{promptbox}
Evaluate the quality of the following search query for a question answering system.\newline
\newline
Context: \{context\}\newline
\newline
Thinking before this query: \{thinking\}\newline
\newline
Generated Query: \{query\}\newline
\newline
IMPORTANT: This is a multi-step reasoning system. The query does NOT need to directly answer the final question in one step. Instead, evaluate whether it makes good progress toward the answer by retrieving useful intermediate information.\newline
\newline
Evaluate this query on these criteria:\newline
1. Relevance: Will it retrieve information that makes progress toward answering the question? (Intermediate steps are valuable!)\newline
2. Specificity: Is it specific enough to get useful results?\newline
3. Searchability: Is it well-formed for a search engine with appropriate keywords? Good queries combine multiple relevant terms.\newline
4. Alignment: Does it align with the thinking step that preceded it?\newline
5. Novelty: Does it explore new information (not redundant with the context)? If the context already contains the answer to what the query is searching for, the query is redundant and unhelpful.\newline
\newline
Provide a score using EXACTLY one of these three values:\newline
- +1: GOOD -- Specific, well-formed query that will retrieve useful information to make progress (even if intermediate). Has clear keywords and good searchability. Combines multiple relevant terms or uses specific names/concepts.\newline
- ~0: ACCEPTABLE -- Query has some specificity but could be improved. May lack context-specific keywords or be somewhat generic, but shows reasonable attempt at targeting the information need.\newline
- -1: BAD -- Single generic word without context (e.g., just ``singer'', ``perfume'', ``city''), completely irrelevant to the question, redundant with information already in the context, or so poorly formed it will return millions of unhelpful results.\newline
\newline
First provide your reasoning, then the score. Use this exact format:\newline
<explanation> Your reasoning here </explanation>\newline
<score> numerical score </score>
\end{promptbox}

\paragraph{Final Answer Reward Prompt.}
\begin{promptbox}
Evaluate if the predicted answer correctly answers the question.\newline
\newline
Context: \{context\}\newline
\newline
Ground Truth Answer: \{ground\_truth\}\newline
\newline
Predicted Answer: \{predicted\_answer\}\newline
\newline
Compare the predicted answer to the ground truth. They don't need to be word-for-word identical, but the predicted answer should convey the same core information.\newline
\newline
Provide a score using EXACTLY one of these three values:\newline
- +1: CORRECT -- The predicted answer conveys the same core information as the ground truth\newline
- ~0: PARTIALLY CORRECT -- The answer is incomplete, ambiguous, or contains minor inaccuracies\newline
- -1: INCORRECT -- The answer is wrong or contradicts the ground truth\newline
\newline
First provide your reasoning, then the score. Use this exact format:\newline
<explanation> Your reasoning here </explanation>\newline
<score> numerical score </score>
\end{promptbox}

\end{document}